%% file: icra-submission.tex
\documentclass[letterpaper, 10 pt, conference]{ieeeconf}  

\IEEEoverridecommandlockouts                              

\makeatletter
\let\NAT@parse\undefined
\makeatother

\def\arxiv{}

\usepackage{graphicx}
\usepackage{multicol}
\usepackage{color}
\usepackage{hyperref}
\definecolor{linkcolor}{rgb}{0.45,0.05,0.05}
\definecolor{citecolor}{rgb}{0.05,0.45,0.45}
\definecolor{urlcolor}{rgb}{0.05,0.05,0.45}
\hypersetup{
  linkcolor  = linkcolor,
  citecolor  = citecolor,
  urlcolor   = urlcolor,
  colorlinks = true,
  bookmarks = true,
}

\usepackage{epsfig} 
\usepackage[mathscr]{euscript}
\usepackage[ruled,vlined,noend,linesnumbered]{algorithm2e}
\usepackage{mathptmx} 
\usepackage{times} 
\usepackage[font=small,labelfont=bf]{caption}
\usepackage{subcaption}
\usepackage{amsmath} 
\usepackage{amsthm}
\usepackage{amssymb}  
\usepackage{tikz}  
\usetikzlibrary{arrows,automata,positioning}
\usepackage{tabularx}
\usepackage{xspace}
\usepackage{siunitx}
\usepackage{mathtools}
\usepackage{bm}
\let\labelindent\relax
\usepackage{enumitem}

\usepackage{listings}
\usepackage{xcolor}
\usepackage{comment}

\input{defs.tex}
\def\BibTeX{{\rm B\kern-.05em{\sc i\kern-.025em b}\kern-.08em
    T\kern-.1667em\lower.7ex\hbox{E}\kern-.125emX}}

\title{A general class of combinatorial filters that can be minimized efficiently
\shortvspace{-0.6ex}
}

\author{Yulin Zhang and Dylan A. Shell%
\thanks{Y. Zhang is with Amazon Robotics, North Reading, MA, USA. 
{\tt\small zhangyl@amazon.com}. The work was done prior to joining Amazon.
D. A. Shell is with Dept. of Computer Science \& Engineering, 
Texas A\&M University, College Station, TX, USA.
{\tt\small dshell@tamu.edu}.
This work was supported by the NSF through awards \href{https://nsf.gov/awardsearch/showAward?AWD_ID=1849249}{IIS-1849249} and \href{https://nsf.gov/awardsearch/showAward?AWD_ID=2034097}{IIS-2034097}. 
}
\shortvspace{-0.6ex}
}

\begin{document}

\maketitle

\begin{abstract}

State minimization of combinatorial filters is  a fundamental problem that
arises, for example, in building cheap, resource-efficient robots.  But exact
minimization is known to be \nphard.  This paper conducts a more nuanced
analysis of this hardness than up till now, and uncovers two factors which
contribute to this complexity.  We show each factor is a distinct source of the
problem's hardness and are able, thereby, to shed some light on the role played by
(1)\,structure of the graph that encodes compatibility relationships, and 
(2)\,determinism-enforcing constraints.
Just as a line of prior work has sought to introduce
additional assumptions and identify sub-classes that lead to practical state
reduction, we next use this new, sharper understanding to explore special cases
for which exact minimization is efficient.  
%
We introduce a new algorithm for constraint repair that applies
to a large sub-class of filters, subsuming three distinct special cases for which 
the possibility of optimal minimization in polynomial time was known earlier.
While the efficiency in each of these three cases previously appeared to stem
from seemingly dissimilar properties, when seen through the lens of the present
work, their commonality now becomes clear.  
We also provide entirely new
families of filters that are efficiently reducible.
\end{abstract}

\section{Introduction}
\label{sec:introduction}

Combinatorial filters are discrete transition systems that process streams of
observations to produce outputs sequentially. They have found
practical application as estimators in multi-agent tracking problems~(e.g.,~\cite{tovar2014combinatorial}) and
as representations of feedback plans/policies for robots (e.g.,~\cite{okane17concise,zhang22sso}). 
Unlike traditional recursive Bayesian filters (the class of estimator most
familiar to roboticists, see~\cite{thrun05probabilistic}), combinatorial
filters allow one to ask questions regarding minimality. 
By reducing their size, one can design resource-efficient robots---a
consideration of practical importance.  More fundamentally, through filter
minimization, one may discover what information is necessary to compute a
particular estimate, or what needs to be tracked in order to have sufficient
knowledge for a given task.
Determining a task's information requirements and limits is a basic problem
with a long history in robotics~\cite{donald2012,lavalle10sensing}, and has begun gaining interest again (e.g.,~\cite{majumdar2022}).
%
%
Unfortunately, given some combinatorial filter, computing the smallest
equivalent filter---its \emph{minimizer}---is an \nphard problem.  

This paper uncovers and examines two different factors which contribute to this
complexity: the first has to do with the structure of the compatibility graph induced by the filter; the
second involves auxiliary (or zipper) constraints added during minimization to
ensure the result will be deterministic.  As we show, both are distinct
dimensions and form independent sources of the problem's hardness.  This is
the first contribution of the paper (and constitutes the subject of Section~\ref{sec:hardness}).

Like most hard problems of practical significance, 
a line of research has
sought
specially structured sub-classes of filters that allow
efficiency to be salvaged\,\cite{saberifar17special}.
Another line  has examined relaxed\,\cite{saberifar17inconsequential} or other
restricted forms of
reduction\,\cite{saberifar18improper,rahmani21equivalence}. 
In the prior work, three particular sub-classes of filter have been identified for which 
optimal filter minimization is known to be possible in polynomial \arxivOrShort{time, or for which 
efficient algorithms have been provided, namely:}{time:}
{($i$)}\,\nme filters\,\cite{saberifar17special},
{($ii$)}\,\oao filters\,\cite{saberifar17special},
and 
{($iii$)}\,\un filters\,\cite{zhang22sso}.

The second portion of the paper, building upon the first, establishes a new sub-class of filters for which exact minimization is achievable in polynomial time. 
This sub-class strictly subsumes those of
($i$), ($ii$), and ($iii$), and also provides some understanding of
why both factors\,---the compatibility graph and auxiliary/zipper constraints---\,are tame for these filters. Part of the answer is that it is possible
to ignore the constraints because they can be repaired afterwards:
Section~\ref{sec:repairable} introduces an algorithm for constraint repair that
applies broadly, including for the new sub-class we study and, hence, 
the three prior ones as well. Another part of the answer, the requirement to quickly generate minimal
clique covers, is feasible for the three prior sub-classes, ($i$)--($iii$), because their compatibility graphs all turn out to be chordal.  Thus,
their apparent distinctiveness happens to be superficial and, in reality, their efficiency stems from some common underlying properties.

\subsection{Context, general related work, and roadmap of the paper}

As the contributions of this work are of a theoretical nature, we leave the
customary motivating settings and example application domains to the work we cite next; each
and every one of the following include specific problem instances, so we 
trust the reader will glance at those papers to allay any doubts as to practical utility.
The term combinatorial filter was coined by Tovar \emph{et al.}~\cite{tovar2014combinatorial}, and
the minimization problem was formulated and its hardness established in~\cite{okane17concise}.
The current state-of-the-art algorithm for 
combinatorial filter minimization was presented at ICRA'21 in~\cite{zhang2021accelerating}.
The starting point for our current treatment is the authors' paper~\cite{zhang20cover}, which showed that filter minimization is 
equivalent to the classic graph problem of determining the minimal clique cover,
when augmented with auxiliary constraints.

The next section will provide necessary definitions and theoretical
background. 
Section~\ref{sec:hardness} first delineates
important sub-families of graphs and uses them to establish our key hardness results.  
Section~\ref{sec:repairable} turns to 
constraints and provides an algorithm to repair constraint-violating
solutions when specific conditions are met.
Thereafter, 
the results are consolidated into a new,
efficiently minimizable sub-class of filters and this is
connected with prior special
sub-classes in  Section~\ref{sec:special-cases}.
\arxivOrShort{
A short summary and conclusion forms the paper's last section.
}{
The final section presents the conclusion. Space considerations have meant that most proofs had to be omitted; the full proofs all appear in~\cite{zhang22efficient}.
}

\shortvspace{-0.1ex}
\section{Preliminaries}
\shortvspace{-0.3ex}
\subsection{Basic definitions}

\arxivOrShort{
\begin{definition}[filter~\cite{setlabelrss}]
A \defemp{deterministic filter}, or just \defemp{filter},
is a 6-tuple $\struct{F} = (V, v_0, Y, \tau, C, c)$, with $V$ a non-empty
finite set of states, $v_0$ an initial state, $Y$ the set of
observations, $\tau: V\times Y \hookrightarrow V$ 
the partial function describing transitions, $C$ the set of outputs, and $c: V\to
C$ being the output function.  
\end{definition}
}{
\begin{definition}[\cite{setlabelrss}]
A \defemp{deterministic filter}, or just \defemp{filter},
is a 6-tuple $\struct{F} = (V, v_0, Y, \tau, C, c)$, with $V$ a non-empty
finite set of states, $v_0$ an initial state, $Y$ the set of
observations, $\tau: V\times Y \hookrightarrow V$ 
the partial function describing transitions, $C$ the set of outputs, and $c: V\to
C$ being the output function.  
\end{definition}
}

Filters process finite sequences of elements from $Y$ in order to produce a
corresponding sequence of outputs (elements of $C$).  Any filter does this by tracing
from $v_0$ along edges (defined by the transition function) and producing outputs (via the $c$ function) as it visits
states.  
States $v_i$ and $v_k$ are 
understood to be connected by a 
directed edge
bearing label $y$,
if and only if $\tau(v_i, y) = v_k$.
We will assume $Y$ to be non-empty, and that no state is unreachable from $v_0$.

%

\smallskip

This paper's central concern is the following problem:

%

\providecommand{\cri}[1]{{\footnotesize\textbf{#1.}}}

\ourproblem{\textbf{Filter Minimization (\fm{\struct{F}})}}
{A deterministic filter $\struct{F}$.}
{A deterministic filter $\struct{F}^{\star}$ with fewest states,~such~that:
\begin{itemize}[leftmargin=9pt,itemindent=2pt]
\item [\cri{1}]\hspace*{-2pt}any sequence which can be traced on $\struct{F}$ can also be traced on $\struct{F}^{\star}$;
\item [\cri{2}]\hspace*{-2pt}the outputs they produce on any of those sequences are identical.
\end{itemize}
\vspace*{-12pt}
\phantom{.}
}

Solving this problem requires some minimally-sized filter~$\struct{F}^{\star}$ that is functionally equivalent to $\struct{F}$, where the 
notion of equivalence\,---called \emph{output simulation}---\,needs only 
criteria~\cri{1} and~\cri{2} to be met. For a formal definition of output simulation, see~\cite[Definition~5, pg.~93]{zhang20cover}.\footnotemark

\footnotetext{After examining filters like those here, the later sections of that paper go further by studying a generalization in which function $c$ may be a relation. Complications arising from that generalization will not be discussed herein.}

\begin{lemma}[\cite{okane17concise}]
\label{lem:FMhardness}
The problem \fmp is \nphard.
\end{lemma}
\shortvspace{-0.75ex}

\shortvspace{-0.25ex}
\subsection{Constrained clique covers on a graph}
\shortvspace{-0.25ex}

In giving a minimization algorithm,
\fmp  was recently connected to an
equivalent graph covering problem\,\cite{zhang20cover}.
To start, we consider this problem abstractly in isolation:

\ourproblem{\textbf{Minimum Zipped Clique Cover} (\mzcc{G}{Z})}
{A graph $G=(V,E)$ and 
collection of
zipper constraints 
$Z=\{(U_1, V_1),
\dots,
(U_m, V_m)\}$, with
$U_i, V_i \subseteq V$.
}
{Minimum cardinality clique cover $\cover{K}$ such that:
\begin{itemize}[leftmargin=9pt,itemindent=2pt,parsep=6pt]
\item [\cri{1}]\hspace*{-6pt}$\bigcup\limits_{K_i \in \cover{K}}\!\!\!\! K_i = V$\!,
 with each $K_i$ forming a clique on~$G$;
\item [\cri{2}]\hspace*{-2pt}$\forall K_i \in \cover{K}$, if there is some $\ell$ such that $U_\ell \subseteq K_i$, then some \mbox{$K_j \in \cover{K}$} must have $K_j \supseteq V_\ell$. 
\end{itemize}
\vspace*{-12pt}
\phantom{.}
}

The constraints in $Z$ are still rather arcane, hence the
next section, in making a connection to \fmp, 
provides an explanation
of how $Z$ is used,
what it is for, and why it bears the moniker \emph{zipper}.

\shortvspace{-0.25ex}
\subsection{Filter minimization as constrained clique covering}
\shortvspace{-0.25ex}

In bridging filters and clique covers, the key idea
is that certain sets of 
states in a filter can be identified as candidates to merge together, and such
`mergability' can be expressed as a graph. The process of forming covers of this graph identifies states to consolidate and, accordingly, minimal covers yield 
small filters. The first technical detail concerns this graph and states that are candidates to be merged:

\shortvspace{-0.75ex}
\begin{definition}[extensions and compatibility]
\label{def:compat}
For a state $v$ of filter $\struct{F}$, we will use $\extension{\struct{F}}{v}$ to
denote the set of observation sequences, or \defemp{extensions}, that can be traced starting from~$v$.  
Two states $v$ and $w$
are \defemp{compatible} with each other if their outputs agree on 
$\extension{\struct{F}}{v} \cap \extension{\struct{F}}{w}$, their
common extensions. In such cases, we will write $v\compatible w$.
The \defemp{compatibility graph} $G_\struct{F}$ possesses edges between states
\arxivOrShort{if and only if}{iff}
 they are compatible.
\end{definition}
\shortvspace{-0.75ex}

However, simply building a minimal cover on $G_\struct{F}$ is not enough because covers 
may merge some elements which, when
transformed into a filter,
produce nondeterminism.
The core obstruction is when 
a fork is created, as when
two
compatible states are merged, both
of which have outgoing edges bearing identical labels, but whose destinations
differ. 
To enforce determinism, we
introduce constraints that 
forbid forking and require mergers
to flow downwards. 
The following specifies such a constraint:

\begin{definition}[determinism-enforcing 
zipper constraint]
Given a set of mutually compatible states $U$ in $\struct{F}$ and the set of all its $y$-children $V=\{v\mid u\in U, \tau(u, y)=v\}$, then the pair $(U, V)$ is a \emph{determinism-enforcing zipper constraint} of $\struct{F}$. 
\end{definition}

A zipper constraint is satisfied by a clique if $U$ is not covered in a clique, or both $U$ and $V$ are covered in cliques. (This is criterion  \cri{2} for \mzccp.)
For filters, in other words, if the states in $U$ are to be consolidated then
the downstream states, in $V$, must be as well.

The collection of all  determinism-enforcing 
zipper constraints for a filter 
$\struct{F}$ is denoted $Z_\struct{F}$.
Both $G_\struct{F}$ and $Z_\struct{F}$ are clearly polynomial in the size of~$\struct{F}$.
Then, a minimizer of $\struct{F}$ can be obtained from the solution to the minimum
zipped vertex cover problem, \mzcc{G_\struct{F}}{Z_\struct{F}}: 
\shortvspace{-0.75ex}
\begin{lemma}[\cite{zhang20cover}]
Any \fm{\struct{F}}\!\! can be converted into an \mzcc{G_\struct{F}}{Z_\struct{F}} in polynomial time; hence \mzccp is \nphard.
\label{lm:fm_eqv}
\end{lemma}
\shortvspace{-1.95ex}
Though we skip the details, the proof in~\cite{zhang20cover} of the preceding also gives an efficient way to construct a deterministic
filter from the minimum cardinality clique cover.

\shortvspace{-0.25ex}
\section{Hardness: Reexamined and Refined}
\label{sec:hardness}
\shortvspace{-0.25ex}

The recasting of \fm{\struct{F}} as \mzcc{G_\struct{F}}{Z_\struct{F}} leads one naturally to
wonder: what precise role do the compatibility graph and zipper constraints play with
regards to hardness?

\shortvspace{-0.75ex}
\subsection{Revisiting the original result}
\shortvspace{-0.25ex}
Firstly, examining the proof of Lemma~\ref{lem:FMhardness}, the argument
in \cite{okane17concise} proceeds by reducing the graph \num{3}-coloring problem to filter minimization. Looking at that construction carefully, one observes that the
\fmp instance that results from any 
\num{3}-coloring problem does not have any zipper constraints. Hence, by writing
\mzccp with compatibility graph of \struct{F} and no zipper constraints as \mcc{G_\struct{F}}, we get the following:

\shortvspace{-0.75ex}
\begin{lemma}
\label{lem:emptyzips}
\mcc{G_\struct{F}} is \nphard.
\end{lemma}
\shortvspace{-0.75ex}
\arxivOrShort{
\begin{proof}
The original 
construction in 
\cite{okane17concise} is sufficient.
\end{proof}}{}
A superficial glance
might cause one to think of \mzccp with an empty collection of zipper constraints as the
standard minimum clique cover problem, {viz.}, \textnumero\;\num{13} of Karp's
original \num{21} \npcomplete problems\,\cite{karp1972reducibility}.  
Actually,
Lemma~\ref{lem:emptyzips} states that the clique cover instances 
arising in minimization of filters are \nphard; note that
this is neither a direct
restatement of Karp's original fact nor merely entailed by it. (But see, also, 
Theorem~\ref{thm:realizable_compatibility} below.)

\shortvspace{-0.3ex}
\subsection{Special graphs: \ec cases}

To begin to investigate problems with special structure, our starting point is to
recognize that several specific sub-families of undirected graphs (some widely known,
others more obscure) allow a minimal clique cover to be obtained
efficiently.
We formalize such cases with the following.

\arxivOrShort{
\begin{definition}[\ec]
A sub-family  of graphs  $\set{G}$ is termed \defemp{\ec} if there is some
algorithm~$\algo{A}_\set{G}$ such that, for all $G \in \set{G}$, 
$\algo{A}_\set{G}(G)$ gives a minimal clique cover of $G$  and does so in polynomial time.
\end{definition}
}{
\shortvspace{-0.3ex}
\begin{definition}
A sub-family  of graphs  $\set{G}$ is \defemp{\ec} if some
algorithm~$\algo{A}_\set{G}$ exists so that, $\forall G \in \set{G}$, 
$\algo{A}_\set{G}(G)$ produces a minimal clique cover of $G$ in polynomial time. 
\end{definition}
}
\shortvspace{-0.3ex}
In filters with \ec compatibility graphs, when also $Z_\struct{F}=\emptyset$, then criterion~{\cri{2}} of
\mzccp holds vacuously and \fm{\struct{F}} will be efficient.
The contrast of this statement with Lemma~\ref{lem:emptyzips}, shows
that the \ec sub-families carve out subsets of easy problems.

Lemmas~\ref{lem:chordal-fwu}, 
\ref{lem:perfect-fwu} and
\ref{lem:triangle-free-fwu}, and Theorem~\ref{thm:combo-fwu}, which will follow, review some instances of \ec graphs:

\arxivOrShort{
\begin{definition}[chordal graph\,\cite{gross2013handbook}]
\label{def:chordal}
A graph is \defemp{chordal} if all cycles of four or more vertices have a
chord, which is an edge not part of the cycle but which connects two vertices in the cycle.
\end{definition}
}{
\shortvspace{-0.3ex}
\begin{definition}[\cite{gross2013handbook}]
\label{def:chordal}
A graph is \defemp{chordal} if all cycles of four or more vertices have a
chord, which is an edge not part of the cycle but which connects two vertices in the cycle.
\end{definition}
\shortvspace{-0.5ex}
}

\shortvspace{-1.3ex}
\begin{lemma}[\cite{gavril72algos}]
\label{lem:chordal-fwu}
Chordal graphs are \ec.
\end{lemma}
\shortvspace{-0.3ex}
\arxivOrShort{
\begin{proof}
A classic algorithm for computing a minimum covering by cliques appears
in~\cite{gavril72algos}, using the existence of a perfect elimination ordering,
it being well known (but first established in~\cite{rose70triangle})
that a graph is chordal if and only if it admits such an
ordering. 
\end{proof}
}{}

\shortvspace{-0.3ex}
A strictly larger class of graphs are those that are perfect.
\arxivOrShort{
\begin{definition}[perfect graph\,\cite{gross2013handbook}]
\label{def:perfect}
A \defemp{perfect graph} is a graph where the chromatic number of every induced
subgraph equals the order of the largest clique of that subgraph.
\end{definition}
}{
\begin{definition}[\cite{gross2013handbook}]
\label{def:perfect}
A \defemp{perfect graph} is a graph where the chromatic number of every induced
subgraph equals the order of the largest clique of that subgraph.
\end{definition}
\shortvspace{-0.5ex}
}

\shortvspace{-1.3ex}
\begin{lemma}[\cite{grotschel1988geometric}]
\label{lem:perfect-fwu}
Perfect graphs are \ec.
\end{lemma}
\arxivOrShort{
\begin{proof}
The minimum clique cover for a perfect graph can be found in polynomial time~
\cite[Theorem~9.3.30, pg.~294]{grotschel1988geometric}.
\end{proof}
}{}
\shortvspace{-0.5ex}



As all chordal graphs are also perfect, Lemma~\ref{lem:chordal-fwu} follows from Lemma~\ref{lem:perfect-fwu}, and
the reader may wonder why then chordal graphs are worth mentioning explicitly. Three reasons: 
(1)~the requirements 
of Definition~\ref{def:chordal} tend to be less demanding to check than those in
Definition~\ref{def:perfect}, which involve some indirectness;
(2)~the polynomial-time algorithm of \cite{grotschel1988geometric} 
(referenced in proof of Lemma~\ref{lem:perfect-fwu}) 
is not a direct combinatorial method and, in fact, 
researchers continue to contribute practical
methods tailored to specific sub-classes of perfect graphs (e.g.,~\cite{bonomo13minimum});
(3)~chordal graphs will show up in the proofs, including in the next section.

But there are graphs, beyond only those which are perfect, that still give \ec problems:

\arxivOrShort{
\begin{definition}[triangle-free graph\,\cite{gross2013handbook}]
A \defemp{triangle-free} graph is an undirected graph where no three vertices
have incident edges forming a triangle.
\end{definition}
}{
\shortvspace{-0.4ex}
\begin{definition}[\cite{gross2013handbook}]
A \defemp{triangle-free} graph is a graph where no three vertices
have incident edges forming a triangle.
\end{definition}
\shortvspace{-1.0ex}
}
A specific triangle-free graph that is not perfect is
the Gr\"{o}tzsch graph.

\shortvspace{-0.4ex}
\begin{lemma}[\cite{edmonds65blossom}]
\label{lem:triangle-free-fwu}
The triangle-free graphs are \ec.
\end{lemma}
\shortvspace{-0.4ex}
\arxivOrShort{
\begin{proof}
A folklore algorithm for computing the minimal clique cover for triangle-free 
graphs is to compute a maximal matching (e.g., via~\cite{edmonds65blossom}), and then treat
unmatched singleton vertices as cliques themselves.
\end{proof}
}{}

Finally, composition allows treatment of graphs with mixed
properties, e.g., 
 we might have filters with compatibility graphs 
where some components are perfect, and others are triangle-free.
The following fact is useful in such cases.

\arxivOrShort{
}
{
}

\begin{theorem}[mix-and-match]
\label{thm:combo-fwu}
Suppose a graph $G = (V, E)$, where $E \subseteq V\times V$,
is made up of components $G_1, G_2, \dots, G_m$, every $G_i = (V_i, E_i)$,
where $V$ is partitioned into mutually disjoint set of
vertices $V_1, V_2, \dots, V_m$, and $E_1, E_2, \dots, E_m$ with every $E_i \subseteq V_i\times V_i$.
If each of the $G_i$'s is \ec, then $G$ is \ec.
\end{theorem}
\arxivOrShort{
\begin{proof}
One applies the algorithm associated with
each component to that component. The union of their results is a minimum clique cover
for $G$ and this only requires polynomial time.
\end{proof}
}{}

Filters naturally yield graphs comprising separate components
as  the output values directly partition the vertices. That is, 
compatibility graphs never possess edges between any vertices $v$ and
$w$ where $c(v) \neq c(w)$.

\subsection{Special compatibility graphs and non-empty zippers}

In light of Lemma~\ref{lem:emptyzips} showing
that zippers are not needed to have hard problems, 
and the fact that there are sub-families of graphs
for which minimal covers may be obtained efficiently, we next ask: do the zipper
constraints themselves contribute enough complexity so that
even with an \ec instance, we can get a hard problem? 

The answer is in the affirmative and we use the sub-family of chordal
graphs to establish this.  We begin with a triangulation procedure that, given
a general graph, constructs one which is chordal. 
The approach to the proof
is to think about solving \mzccp on the chordal version and then relate the
solution back to the original problem.


A graph is non-chordal if and only if there is an $m$-cycle, with $m \geq 4$ and there is no edge cutting across the cycle.
We
can break such a cycle into smaller ones by adding edges as shortcuts.  Repeating this
process will triangulate such cycles and the procedure must eventually
terminate as the complete graph is chordal.  We call these newly introduced edges
\emph{dashed}  as this is how we shall depict them visually.

Having introduced
extra edges, the idea is to discourage clique covers from ever choosing to
cover any of these new dashed edges via penalization. 
A penalty is incurred by
being compelled to choose additional cliques---zipper constraints are rich
enough to force such choices.
This requires the introduction of a gadget we term a
`necklace'. Suppose the original non-chordal graph $G$ had $n$ vertices,
$m_s$ edges, and that an additional $m_d$ dashed edges were used to triangulate
the graph.  Then, as illustrated in Figure~\ref{fig:chordalization}, we first
make $\ell=m_s+m_d+1$ copies of $2$-vertex connected graphs, which we dub
`pendants'. 
These are laid in a line, and
between any pair of pendants, we place a single black vertex that we call a `bead'.
The $\ell$ pendants and $\ell-1$ beads are strung together via edges, each bead being connected to the two 
adjacent pendants. We'll call these connecting edges `strings'. 
To each dashed edge we add $2\ell -2 = 2(m_s+m_d)$ zipper constraints,
connecting the dashed edge to the length of strings.  This
construction means that when a dashed edge is covered, its zipper constraints
become active and then each bead will have to appear in two separate covers, one for each neighboring pendant.


Given graph and zipper constraints $(G,Z)$, 
we will denote the result of the construction just described  with $(\chord{G},\neckl{Z})$, the first element being the chordal graph along with the necklace, and the second element being the additional constraints.
Then: 
(1)~$\chord{G}$ is chordal (dashed edges made $G$ chordal, the necklace itself is chordal), 
(2)~$G \subseteq \chord{G}$, as vertices/edges were added, never removed, 
(3)~$Z \subseteq \neckl{Z}$, as constraints were added, not
removed. Further, notice that $\chord{G}$ and $\neckl{Z}$ are no larger than some polynomial
factor of $n$. This construction takes $O(n^4)$.
The purpose of this construction is the following:

\begin{lemma}
Given any non-empty graph $G$ and zipper constraints $Z$, 
a solution 
to \mzcc{G}{Z} can be obtained from any
solution $\chord{S}$ to \mzcc{\chord{G}}{\neckl{Z}}, 
by restricting $\chord{S}$ to only those covers on the vertices of $G$.
\label{lm:chordalization_eqv}
\end{lemma}


\arxivOrShort{
\begin{proof}
To cover the necklace (top half of Figure~\ref{fig:chordalization}) without any zipper constraints being active, 
at least $2m_s+2m_d+1$ cliques will be required: there are $m_s+m_d$ beads, and the
lower vertex of each pendant must be covered, there are a total of $m_s+m_d+1$ 
of these, and they 
 cannot occupy the same cover as the beads: 
hence $2m_s+2m_d+1$ is a lower bound. 
In contrast, if a zipper constraint is triggered, then
to cover the necklace, at
least $2m_s+2m_d$ cliques must cover the edges comprising the string,
and $m_s+m_d+1$ are required to cover the remaining (lower) vertices in
the pendants.  
The difference between a zipper constraint being triggered ($3m_s + 3m_d + 1$) versus not ($2m_s + 2m_d + 1$)  is 
 a penalty of $m_s+m_d$.

Given minimal cover $\chord{S}$, now suppose that $\chord{S}$ groups
the vertices connected by some dashed edge into the same clique, and 
thereby triggered a zipper constraint. Then at least 
$(3m_s+3m_d+1)+1$ cliques are needed (the value in parentheses are required just for the necklace, the extra $1$ accounts for the bottom half).
But consider the trivial cover $S_\triv$: 
for each pair connected by an edge of the original graph, add the pair as a clique, 
cover each pendant as a clique containing the pair,
and then 
cover the beads with singleton cliques.
Then  $|S_\triv| = (2m_s + 2m_d + 1) + m_s$,  but then \mbox{$|S_\triv| < |\chord{S}|$}, as
$(2m_s + 2m_d + 1) + m_s  <  (3m_s + 3m_d + 1) + 1  < |\chord{S}|$ since $ m_s < ( m_s +  m_d + 1)$.
The requirement that $\chord{S}$ trigger a zipper constraint is a contradiction, thus
$\chord{S}$ never chooses to cover any dashed edges, and the restriction\,---ignoring covers for the top half---\,gives a cover for $G$. Denote this restriction $\chord{S}_{\!\scalebox{0.6}{G}}$. Suppose that 
$S$ is a solution to \mzcc{G}{Z} with 
$|S| < |\chord{S}_{\!\scalebox{0.6}{G}}|$, but then
replacing $\chord{S}_{\!\scalebox{0.6}{G}}$ with $S$ would yield a smaller solution to 
\mzcc{\chord{G}}{\neckl{Z}} than $\chord{S}$.
Hence, $|\chord{S}_{\!\scalebox{0.6}{G}}| = |S|$, as required.
\end{proof}
}
{
\vspace{-2pt}
This is one of the paper's main results: the proof calculates the penalty incurred by covering any dashed edge (i.e., the additional covers for part of the necklace) and shows it to be so large that even covering non-dashed edges in pairs would be preferable. Any optimal cover would do no worse than this and, thus, a zipper constraint can never
be triggered.
}

\setlength{\belowcaptionskip}{-15pt}
\begin{figure}
\centering
\includegraphics[width=\linewidth]{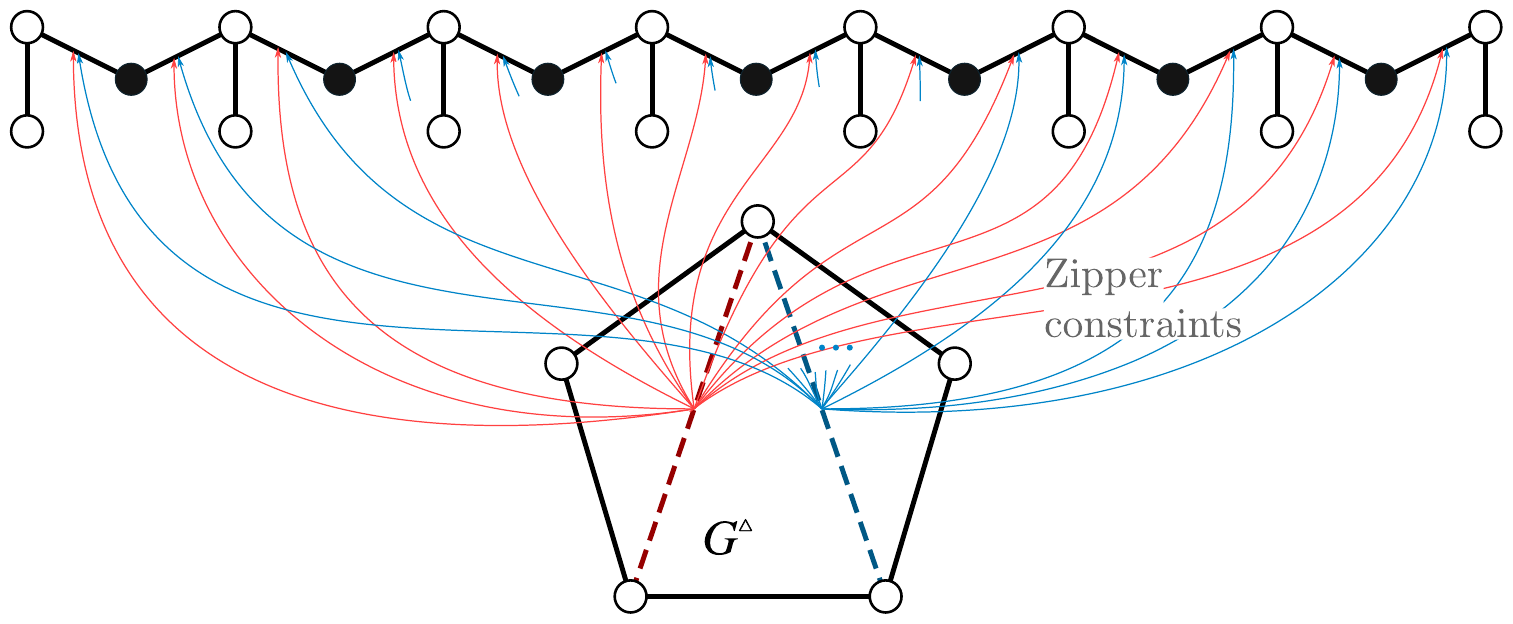}
\caption{A clique cover problem on a general graph is reduced to a clique cover
problem on a chordal graph with extra zipper constraints. 
Two dashed edges  are
added to the pentagon in order to triangulate it, 
resulting in a 
$7$-edge 
chordal
graph.
Dashed edges are made undesirable 
through the addition of 
zipper constraints that trigger the necklace string (at top). Zipper constraints are represented as arrows,
shown in red and blue to associate them visually with their dashed edge. (Note:
parts of some blue arrows have been
elided to reduce visual clutter).\label{fig:chordalization}} 
\end{figure}


\vspace{-2pt}
\begin{theorem}
\mzcc{G_\ecsimple}{Z}, where $G_\ecsimple$ is \ec and 
$Z \neq \emptyset$ is \nphard.
\end{theorem}
\arxivOrShort{
\begin{proof}
One needs to reduce from a known
\nphard problem to an instance of
\mzcc{G_\ecsimple}{Z}, where $G_\ecsimple$ is \ec.
Leveraging 
Lemma~\ref{lem:emptyzips},
consider any instance of
\mcc{G_\struct{F}}. 
Then use Lemma~\ref{lm:chordalization_eqv} (taking $G_\struct{F}$ for $G$ and $\emptyset$ for $Z$), to  
obtain an 
$\chord{G}_\struct{F}$
and
$\chord{Z}$, where 
$\chord{G}_\struct{F}$
is chordal and \mzcc{\chord{G}_\struct{F}}{\chord{Z}} is an 
equivalent problem.
But $\chord{G}_\struct{F}$ is \ec (Lemma~\ref{lem:chordal-fwu}), hence the necessary reduction is complete.
\end{proof}
}{
\shortvspace{-0.8ex}
The proof uses 
Lemma~\ref{lem:emptyzips}, for a known \nphard problem, in order to obtain a reduction via
Lemma~\ref{lm:chordalization_eqv}.
\shortvspace{0.2ex}
}

The preceding leads to the following interpretation. If a filter \struct{F} induces an \ecadj $G_\struct{F}$ then:
\begin{itemize}
\item when $Z_\struct{F} = \emptyset$, since 
criterion~{\cri{2} of
\mzccp} holds vacuously, \fm{\struct{F}} is in \p;
\item when $Z_\struct{F} \neq \emptyset$,  despite $G_\struct{F}$ being \ec, \fm{\struct{F}} should be suspected as intractable because it is \nphard in the worst case.
\end{itemize}

Additional support for this link between the filter structure, \fmp, \mzccp, and worst-case intractability is the following:

\begin{theorem}
A graph $G$ can be realized as the compatibility graph of some filter if and
only if $G$ either: 
(1) has at least two connected components, or
(2) is a complete graph.
\label{thm:realizable_compatibility}
\end{theorem}
\arxivOrShort{
\begin{proof}
$\implies$ We form a filter $\struct{F}$ with a state for each vertex in~$G$.
First, identify the set of connected components in $G$;
suppose there are $n$ of them. Taking the set of outputs $C=\{1,\dots,n\}$,
have the output function, $c:V\to C$, give each state the number of the 
component to which it belongs.
Next, for any vertices $v$
and $w$ in $G$ that are not directly connected by an edge, 
suppose their corresponding states in $\struct{F}$ are
$s_v$ and $s_w$, respectively. Then add edges from $s_v$ and $s_w$ 
to states $s'_v$ and $s'_w$, chosen
arbitrary but with $c(s'_v) \neq c(s'_w)$;
label those two edges with the same observation $y_{(v,w)}$. 
When $n\geq 2$, the preceding will have made all such $s_v$ and $s_w$ disagree on a common
extension in $\struct{F}$, thus, ensuring $s_v \not\compatible s_w$.
When $n=1$ and $G$ is complete, no fitting $v$ and $w$ exist, so no edges will have been added to 
$\struct{F}$.
Finally, we pick an arbitrary state as an
initial state, and adjoin an edge from the initial state to each of the
other states in the filer,
labeling each edge with an observation unique to it. 
In the resulting $\struct{F}$, if two states do not
share the same common outgoing observation, then they are compatible. 
Vacuously, when $n=1$ with complete $G$, all states must be mutually compatible.
Therefore, $\struct{F}$ must have $G$ as its compatibility graph, $G_\struct{F}$.

$\impliedby$
It remains to show that any
single-component graph that is not also a complete graph can
never arise as the compatibility graph of any filter.
Any edges connecting two vertices in a compatibility graph must
involve states that have the same output value. Thus, to generate a
graph with a single component, all states must have identical outputs.
But then any pair of states must be compatible because their common
extensions can only produce identical sequences---just repetitions of the same output. 
Any graph that is not complete has some pair of vertices not connected by an edge,
which has to come from incompatible states, a possibly that has just 
been precluded.
\end{proof}
}{
The construction 
hinges on assigning
unique outputs to each connected component.
Thereafter, incompatibilities
are easily manufactured in a filter by simply introducing new observation symbols that go to states with different outputs.
}

\begin{corollary}
Chordal graphs with zippered necklace structures, $\chord{G}$,
as in the construction
involved in Lemma~\ref{lm:chordalization_eqv}, are realizable from filters.
\end{corollary}
\arxivOrShort{
\begin{proof}
The triangulation of the original graph $G$ has at least one connected component,
and the necklace forms another, 
thence condition (1) in Theorem~\ref{thm:realizable_compatibility} applies.
\end{proof}
}{}

\section{Repairable Zipper Constraints}
\label{sec:repairable}

If general zipper constraints introduce enough complexity that the problem is
hard even when the graph is \ec, and yet absence of zipper constraints gives
an easy problem, how do we obtain a more discriminating conception 
of zipper constraints and their
structure?
And, specifically, are there special cases of filters which give `nice' zipper
constraints? 
In this section, we first formalize sufficient conditions in order to ignore 
zipper constraints; in these cases once a clique cover
has been obtained, we can modify it to make the
zipper constraints hold. This 
modification step can repair, in polynomial time, any 
zipper-constraint-violating clique cover
for which the sufficient conditions are met and, crucially, can do this without causing any increase in size.



In any graph $G=(V,E)$, 
we refer to the neighbors of a vertex $v\in V$
by set {$\neighbor{G}{v} \coloneqq \{ w\in V\;|\;(w,v) \in E\} \cup \{v\}$}. 
Note that we explicitly include $v$ 
in its own  neighborhood.


Through neighborhoods,
the following condition now describes
a type of 
harmony between zipper constraints and the 
compatibility relation.


\begin{definition}[\cpzn]
\label{def:cpzn}
Given a graph $G=(V,E)$, 
it has the \defemp{\cpzn} property with
respect to an associated collection of zipper constraints $Z = \{ (U_1,W_1),  (U_2,W_2), \dots, (U_m,W_m)\}$,
where $U_k, W_k \subseteq V$, 
if and only if
every pair of vertices $\left\{\zwe{i}{0},\zwe{i}{1}\right\}=W_i$
satisfies either
$\neighbor{G}{\zwe{i}{0}}\subseteq \neighbor{G}{\zwe{i}{1}}$ or  $\neighbor{G}{\zwe{i}{0}}\supseteq \neighbor{G}{\zwe{i}{1}}$.
\end{definition}


The preceding definition
is a sufficient condition  to
yield \mzccp problems whose zipper constraints may be repaired. 

\begin{lemma}
\label{lm:repair}
Let filter $\struct{F}$'s
compatibility graph $G_\struct{F} = (V_{\struct{F}},E_{\struct{F}})$
 possess
the \cpzn property with respect to 
$Z_\struct{F}$.
Suppose cover
$\cover{K}$ 
of $G_\struct{F}$
violates $Z_\struct{F}$.
%
%
Then one may obtain, in polynomial time,
a cover $\cover{M}$
that will satisfy
$Z_\struct{F}$ and has
$|\cover{M}|\leq |\cover{K}|$.
\end{lemma}

\newcommand{\lv}[1]{\ensuremath{{v_{#1}}^{{\!\!\!\!}^{\!{}_{(\ell)}}}\!}}

\arxivOrShort{
\begin{proof}
Given cover $\cover{K} = \{ K_1, K_2, \dots, K_k\}$,
the proof constructs $\cover{M} = \{ M_1, M_2, \dots, M_{m}\}$ such that all zipper constraints will be satisfied through $\cover{M}$, with $m \leq k$.
For $i\in\{1,\dots, k\}$, form the collection of sets $M_{i}$:
\vspace*{-4pt}
\begin{align}
\label{eq:repair}
M_i \coloneqq K_i \cup \bigg\{ \zweBig{a}{\widehat{n}} \;\bigg\vert\; &
\Big(U_a,\left\{\zwe{a}{0}, \zwe{a}{1}\right\}\!\Big) \in Z_\struct{F}, n\in\{0,1\},\\[-6pt] & \;\;\zweBig{a}{n}\in K_i,
\neighbor{G_{\struct{F}}}{\zweBig{a}{n\phantom{\widehat{n}}}\!\!\!\!\!} \subseteq \neighbor{G_{\struct{F}}}{\zweBig{a}{\widehat{n}}} 
\bigg\}, \nonumber
\end{align}
where $\widehat{n} \coloneqq 1 -n$.
Essentially,
to each $K_i$ we are
adding some extra elements.
When $K_i$ contains an element $\zweBig{a}{n}$ that is paired with 
$\zweBig{a}{\widehat{n}}$,
another vertex in the zipper constraints,
then we include $\zweBig{a}{\widehat{n}}$ if it
possesses a neighborhood that is no smaller than $\zweBig{a}{n}$'s.
After doing this for all $K_1, \dots, K_{k}$, we obtain the collection $\cover{M}=\{M_1,\dots, M_{m}\}$ with $m \leq k$ (and $m < k$ only if some  sets grew to become identical).

All vertices previously covered are still covered. 
Moreover, each $M_i$ is a clique because $K_i$ is a clique, and one
can argue inductively as if the extra elements were added sequentially:
As each element $\zweBig{a}{\widehat{n}}$
is added, it is compatible
with element $\zweBig{a}{n}$ already in the clique (owing to the pair being a zipper constraint target). Element  $\zweBig{a}{\widehat{n}}$
has a compatibility neighborhood at least as large as that of $\zweBig{a}{n}$'s, so is compatible with the clique, and the composite thus forms a clique itself.

It only remains to show that all $Z_\struct{F}$ is now satisfied. 
Consider any constraint $(U_i, W_i)$, 
with $W_i = \left\{\zwe{i}{0},\zwe{i}{1}\right\}$, 
then we show that  both $\zwe{i}{0}$ and $\zwe{i}{1}$ appear together in some element in \cover{M}.
The \cpzn property means
there is an $n \in \{0,1\}$ such that 
$\neighbor{G_{\struct{F}}}{\zweBig{i}{n\phantom{\widehat{n}}}\!\!\!\!\!} 
\subseteq \neighbor{G_{\struct{F}}}{\zweBig{i}{\widehat{n}}}$.
Since $\cover{K}$ is a cover, at least one $K_j$ 
exists such that
$\zwe{i}{n}\in K_j$. But then,
 via~\eqref{eq:repair},
$M_j$ must include
 both  $\zwe{i}{n}$ and  $\zwe{i}{\widehat{n}}$.
\end{proof}
}{
\noindent \emph{Repair algorithm and proof.}
Given cover $\cover{K} = \{ K_1, K_2, \dots, \allowbreak K_k\}$,
the proof constructs $\cover{M} = \{ M_1, M_2, \dots, M_{m}\}$ such that all zipper constraints will be satisfied through $\cover{M}$, with $m \leq k$.
For $i\in\{1,\dots, k\}$, form the sets $M_{i}$:
\vspace*{-4pt}
\begin{align}
\label{eq:repair}
M_i \coloneqq K_i \cup \bigg\{ \zweBig{a}{\widehat{n}} \;\bigg\vert\; &
\Big(U_a,\left\{\zwe{a}{0}, \zwe{a}{1}\right\}\!\Big) \in Z_\struct{F}, n\in\{0,1\},\\[-6pt] & \;\;\zweBig{a}{n}\in K_i,
\neighbor{G_{\struct{F}}}{\zweBig{a}{n\phantom{\widehat{n}}}\!\!\!\!\!} \subseteq \neighbor{G_{\struct{F}}}{\zweBig{a}{\widehat{n}}} 
\bigg\}, \nonumber
\end{align}
where $\widehat{n} \coloneqq 1 -n$.
Essentially,
to each $K_i$ we are
adding some extra elements.
When $K_i$ contains an element $\zweBig{a}{n}$ that is paired with 
$\zweBig{a}{\widehat{n}}$,
another vertex in the zipper constraints,
then we include $\zweBig{a}{\widehat{n}}$ if it
possesses a neighborhood that is no smaller than $\zweBig{a}{n}$'s.
After doing this for all $K_1, \dots, K_{k}$, we obtain the collection $\cover{M}=\{M_1,\dots, M_{m}\}$ with $m \leq k$ (and $m < k$ only if some  sets grew to become identical).

All vertices previously covered are still covered. 
Moreover, each $M_i$ is a clique because $K_i$ is a clique, and one 
can argue inductively as if the extra elements were added sequentially:
As each element $\zweBig{a}{\widehat{n}}$
is added, it is compatible
with element $\zweBig{a}{n}$ already in the clique (owing to the pair being a zipper constraint target). Element  $\zweBig{a}{\widehat{n}}$
has a compatibility neighborhood at least as large as that of $\zweBig{a}{n}$'s, so is compatible with the clique, and the composite thus forms a clique itself.

It only remains to show that all $Z_\struct{F}$ is now satisfied. 
Consider any constraint $(U_i, W_i)$, 
with $W_i = \left\{\zwe{i}{0},\zwe{i}{1}\right\}$, 
then we show that  both $\zwe{i}{0}$ and $\zwe{i}{1}$ appear together in some element in \cover{M}.
The \cpzn property means
there is an $n \in \{0,1\}$ such that 
$\neighbor{G_{\struct{F}}}{\zweBig{i}{n\phantom{\widehat{n}}}\!\!\!\!\!} 
\subseteq \neighbor{G_{\struct{F}}}{\zweBig{i}{\widehat{n}}}$.
Since $\cover{K}$ is a cover, at least one $K_j$ 
exists such that
$\zwe{i}{n}\in K_j$. But then,
 via~\eqref{eq:repair},
$M_j$ must include
 both  $\zwe{i}{n}$ and  $\zwe{i}{\widehat{n}}$.
 \qed
 \vspace*{6pt}
}

Notice that the scope of Definition~\ref{def:cpzn} includes
graphs and zipper constraints generally, while
Lemma~\ref{lm:repair} concerns compatibility graphs and zipper constraints obtained specifically from filters, and the
additional structure inherited from the filter shows up in the proof itself.

\section{Special cases: efficiently reducible filters}
\label{sec:special-cases}

Up to this point, sufficient conditions have been presented for favorable \mzccp problem instances.
Each condition concerns a separate factor: 
the first, in Section~\ref{sec:hardness}, deals with structural properties of graphs; while in Section~\ref{sec:repairable},
the second involves the zipper constraints being accordant with neighborhoods of potentially zipped vertices.
The two reflect different dimensions and, as argued above, form distinct sources of the problem's hardness.  
To consolidate---

\shortvspace{-0.75ex}
\ourclass{A 
sub-class of filters that can be minimized efficiently}{
Any filter with \ec compatibility graph and \cpzn can be minimized efficiently.
}

\shortvspace{-0.75ex}
One first constructs the compatibility graph, 
finds any minimum clique cover, and then repairs it using Lemma~\ref{lm:repair}. 

\arxivOrShort{


The requirements for this sub-class, however, pertain to products derived from filters. 
To avoid properties that
must be verified indirectly,
as they involve intermediate products,
we will now
give closer scrutiny to filters themselves.

}
{

Notice that to verify membership of this sub-class 
involves
requirements 
pertaining to products derived from filters; 
we will now
give closer scrutiny to filters themselves.
}

\subsection{Handy properties that yield efficiently reducible filters}
\label{subsec:handy}

We seek properties that are recognizable and verifiable on
filters directly.
To start with, here is a sufficient condition:

\arxivOrShort{
\begin{definition}[\glc]
\label{def:globally-lang-comparable}
A filter $\struct{F}$ is 
\emph{\glc} if, for every pair of compatible states $v \compatible w$,
either extensions
$\extension{\struct{F}}{v} \subseteq \extension{\struct{F}}{w}$ or extensions
$\extension{\struct{F}}{v} \supseteq \extension{\struct{F}}{w}$.
\end{definition}
}{
\begin{definition}
\label{def:globally-lang-comparable}
A filter $\struct{F}$ is 
\emph{\glc} if, for every pair of compatible states $v \compatible w$,
either extensions
$\extension{\struct{F}}{v} \subseteq \extension{\struct{F}}{w}$ or extensions
$\extension{\struct{F}}{v} \supseteq \extension{\struct{F}}{w}$.
\end{definition}
}

\arxivOrShort{
A lemma will make use of the next property:

\begin{property}[condition for transitivity]
\label{prop:transitivity1}
Given three states $u, v, w$ in $\struct{F}$ such that $u\compatible v$ and
$v\compatible w$, we have $u\compatible w$ if they satisfy 
one of the following conditions:
\begin{enumerate}
\item %
$\extension{\struct{F}}{u}\supseteq \extension{\struct{F}}{v}\supseteq \extension{\struct{F}}{w}$\label{condition:one},
\item 
$\extension{\struct{F}}{v}\supseteq \extension{\struct{F}}{u}$ and
$\extension{\struct{F}}{v}\supseteq \extension{\struct{F}}{w}$. \label{condition:two} 
\end{enumerate}
\end{property}

\begin{proof}
To show 1), state $u$ must agree with $w$ on all common extensions: $\extension{\struct{F}}{u} \cap \extension{\struct{F}}{w} = \extension{\struct{F}}{w}$. 
Since
$u \compatible v$ means $u$ agrees with $v$ on
$\extension{\struct{F}}{u} \cap \extension{\struct{F}}{v} \supseteq
\extension{\struct{F}}{w}$.
And $v \compatible w$, so $v$ agrees with $w$ on 
$\extension{\struct{F}}{v} \cap \extension{\struct{F}}{w} =
\extension{\struct{F}}{w}$.
But each agreeing with $v$ 
on $\extension{\struct{F}}{w}$,
means they must agree mutually too.
For 2),
since $u \compatible v$, both $u$ and $v$ agree on extensions $\extension{\struct{F}}{u} \cap \extension{\struct{F}}{v}$; similarly
      $w \compatible v$ means $w$ and $v$ agree on extensions $\extension{\struct{F}}{w} \cap \extension{\struct{F}}{v}$.
To establish the desired fact, 
state $u$ must agree with $w$ on all common extensions, but $\extension{\struct{F}}{u} \cap \extension{\struct{F}}{w} 
= \big(\extension{\struct{F}}{u} \cap \extension{\struct{F}}{v}\big) \bigcap \big(\extension{\struct{F}}{w} \cap \extension{\struct{F}}{v}\big)$, so
each of them agreeing on $v$ gives enough coverage to agree mutually.
\end{proof}

The first use of this property is in the next lemma, showing that filters that  are \glc have 
compatibility graphs that are \ec.

\begin{lemma}\label{lm:glc_chordal}
If $\struct{F}$ is any \glc filter, 
then its compatibility graph $G_{\struct{F}}$ is chordal.
\end{lemma}

\begin{proof}
Chordal graphs have no cycle with size \num{4} or larger; to be chordal, for
any two edges that shared a vertex, $v$ say, between $u$ and $v$, and $v$ and
$w$, it is enough to show there must be a chord connecting $u$ and $w$.
But, for any \glc filter, the 
condition in Property~\ref{prop:transitivity1} holds,
so we know $u \compatible w$.
\end{proof}

}
{
Filters that  are \glc have 
compatibility graphs that are \ec.
\begin{lemma}\label{lm:glc_chordal}
If $\struct{F}$ is any \glc filter, 
then its compatibility graph $G_{\struct{F}}$ is chordal.
\end{lemma}
}

In fact, \glc filters also have benign zipper constraints.

\begin{lemma}\label{lm:glc_to_cpzn}
If $\struct{F}$ is a \glc filter,
then compatibility graph $G_\struct{F}$
 possesses  
the \cpzn property with respect to 
zipper constraints $Z_\struct{F}$.
\end{lemma}
\arxivOrShort{
\begin{proof}
For any pair of states $v$, $w$ and zipper constraint set $Z_\struct{F}=\{(U_1, W_1), (U_2,
W_2), \dots, (U_m, W_m)\}$, if $v\not\compatible w$, then there exists no $W_i$ such
that $\{v, w\} = W_i$ owing to the construction of $Z_\struct{F}$. 
If $v\compatible w$,
since $\struct{F}$ is \glc,
assume without loss of generality that
$\extensions{F}{v}\subseteq \extensions{F}{w}$. 
For any state $u$ with $u\compatible w$,
as $\extensions{F}{u}\cap\extensions{F}{w} \supseteq \extensions{F}{u}\cap\extensions{F}{v}$, 
$u$ and $v$ agree on their common extensions, i.e. $u\compatible v$.
Hence, $\neighbor{F}{v}\supseteq \neighbor{F}{w}$. 
And Definition~\ref{def:cpzn}, thus, holds for $G_\struct{F}$ with respect to $Z_\struct{F}$.
\end{proof}
}{}

\begin{theorem}
\Glc filters can be minimized efficiently.
\end{theorem}
\arxivOrShort{
\begin{proof}
One puts Lemmas~\ref{lm:glc_chordal} and~\ref{lm:glc_to_cpzn} together: the former
tells us, following Lemma~\ref{lem:chordal-fwu}, that a minimum clique cover can be obtained quickly;
the latter, via Lemma~\ref{lm:repair}, says that cover can then be repaired to satisfy
$Z_\struct{F}$.
\end{proof}
}
{
Here, the proof just puts Lemmas~\ref{lm:glc_chordal} and~\ref{lm:glc_to_cpzn} together. 
}

\medskip

\arxivOrShort{
A different but also potentially useful way to identify special cases is to use the
compatibility relation: for a particular problem instance, one might determine
whether specific algebraic properties hold.

As Definition~\ref{def:compat} involves common extensions,
all compatibility relations on states  will
be reflexive ($v \compatible v$, any state being compatible with itself)  and
symmetric ($v \compatible w\, \iff\, w \compatible v $). But some 
relations may have additional properties:

\begin{definition}[equivalence]
\label{def:equivalence}
If filter $\struct{F}$ induces a compatibility relation on  states \compatible\ that is transitive, it is an equivalence relation, and termed a \defemp{compatibility equivalence}.
\end{definition}

Next, is a property involving neighborhoods.

\begin{definition}[\nc]
\label{def:neighborhood-comparable}
A filter $\struct{F}$ is \emph{\nc} if and only if every pair of vertices
$v,w$ in $G_{\struct{F}}$ with $\neighbor{G_\struct{F}}{v}\cap \neighbor{G_\struct{F}}{w} \neq
\emptyset$, satisfies either
$\neighbor{G_\struct{F}}{v}\subseteq \neighbor{G_\struct{F}}{w}$ or  $\neighbor{G_\struct{F}}{v}\supseteq \neighbor{G_\struct{F}}{w}$.
\end{definition}

Though somewhat disguised, the two preceding definitions express the identical concepts.

\begin{property}
\label{prop:transitivity}
A filter $\struct{F}$ is \nc if and only if it induces a compatibility equivalence.
\end{property}
\begin{proof}
$\implies$ (via the contrapositive)\qquad
Suppose $u\compatible v$ and $v\compatible w$, but $u\not\compatible w$.  
Thus, $u \in \neighbor{G_\struct{F}}{u}$ and $u \not\in \neighbor{G_\struct{F}}{w}$, 
and $w \in \neighbor{G_\struct{F}}{w}$ and $w \not\in \neighbor{G_\struct{F}}{u}$. 
Then $\struct{F}$ is not \nc as
$\{v\} \subseteq \neighbor{G_\struct{F}}{u} \cap \neighbor{G_\struct{F}}{w}$ and also, simultaneously,
$\neighbor{G_\struct{F}}{u}\not\subseteq \neighbor{G_\struct{F}}{w}$
and $\neighbor{G_\struct{F}}{u}\not\supseteq \neighbor{G_\struct{F}}{w}$.

\smallskip
$\impliedby$ (also via the contrapositive)\qquad 
Supposing $\struct{F}$ is not \nc, 
there must be a pair of vertices $u$ and $w$ in $G_{\struct{F}}$ such that
$\{v\} \subseteq \neighbor{G_\struct{F}}{u}\cap \neighbor{G_\struct{F}}{w}$,
for some $v$, 
along with $\neighbor{G_\struct{F}}{u}\not\subseteq \neighbor{G_\struct{F}}{w}$ and $\neighbor{G_\struct{F}}{u}\not\supseteq \neighbor{G_\struct{F}}{w}$.
The first fact means $w\compatible v$ and $v\compatible u$.
The second fact means some vertex $z$ exists so that either 
\begin{enumerate}
\item [1)] %
$z \in \neighbor{G_\struct{F}}{u}$ and $z \not\in \neighbor{G_\struct{F}}{w}$, so
$u\compatible z$, $w\not\compatible z$. 
Assuming relation $\compatible$ to be transitive produces a contradiction because 
$w \compatible u$, and $u \compatible z$, but $w \not\compatible z$.
\item [2)]%
$z \not\in \neighbor{G_\struct{F}}{u}$ and $z \in \neighbor{G_\struct{F}}{w}$, 
with an analgous argument shows
$\compatible$ fails to be transitive.
\end{enumerate}
\end{proof}
}
{
A different but also potentially useful way to identify special cases is to use the
compatibility relation: for a particular problem instance, one might determine
whether specific algebraic properties hold. To get there,
first we start with a characteristic 
involving neighborhoods.

\begin{definition}
\label{def:neighborhood-comparable}
A filter $\struct{F}$ is \emph{\nc} 
\arxivOrShort{if and only if}{iff}
every pair of vertices
$v,w$ in $G_{\struct{F}}$ with $\neighbor{G_\struct{F}}{v}\cap \neighbor{G_\struct{F}}{w} \neq
\emptyset$, satisfies either
$\neighbor{G_\struct{F}}{v}\subseteq \neighbor{G_\struct{F}}{w}$ or  $\neighbor{G_\struct{F}}{v}\supseteq \neighbor{G_\struct{F}}{w}$.
\end{definition}

Next, we connect it with an
algebraic property.

\begin{property}
\label{prop:transitivity}
A filter $\struct{F}$ is \nc 
\arxivOrShort{if and only if}{iff}
it induces a compatibility
relation \compatible\ 
on states 
such that \compatible\ 
is an equivalence relation.
\end{property}

}


Further, the \nc property induces an \ec compatibility graph: 
\begin{lemma}\label{lm:nc_chordal}
Any  \nc filter has a compatibility graph that is chordal.
\end{lemma}
\arxivOrShort{
\begin{proof}
In the compatibility graph,
each equivalence class of states forms a disjoint component, every component being itself a complete graph.
Such graphs are called cluster graphs, and 
are a special sub-class of chordal graphs~\cite{gross2013handbook}.
\end{proof}
}
{}

Also we can show that a \nc filter 
has zipper constraints that are repairable.
\begin{lemma}\label{lm:nc_to_cpzn}
Any \nc filter $\struct{F}$
has compatibility graph $G_\struct{F}$
that 
 possesses  
the \cpzn property with respect to 
zipper constraints $Z_\struct{F}$.
\end{lemma}
\arxivOrShort{
\begin{proof}
The construction process that gives the 
zipper constraints $Z_\struct{F}=\{(U_1, W_1), (U_2,
W_2), \dots, (U_m, W_m)\}$ only places pairs of compatible states in
the sets $W_i$, as both vertices are downstream $y$-children of compatible parents.
Suppose $\{v, w\} = W_i$ then we have
$\{v, w\} \subseteq (\neighbor{G_\struct{F}}{v}\cap\neighbor{G_\struct{F}}{w})$ since $v\compatible w$.
Filter $\struct{F}$
being \nc implies 
either
$\neighbor{G_\struct{F}}{v}\subseteq \neighbor{G_\struct{F}}{w}$ or $\neighbor{G_\struct{F}}{v}\supseteq
\neighbor{G_\struct{F}}{w}$, which satisfies the requirements for Definition~\ref{def:cpzn}.
\end{proof}
}

\begin{theorem}
\Nc filters can be minimized efficiently.
\end{theorem}
\arxivOrShort{
\begin{proof}
Lemma~\ref{lm:nc_chordal} (along with Lemma~\ref{lem:chordal-fwu}) say that a minimum clique cover can be obtained efficiently,
and Lemmas~\ref{lm:nc_to_cpzn} and~\ref{lm:repair} mean it can be repaired to satisfy
$Z_\struct{F}$.
\end{proof}
}{}

The properties described in the section\,---one on the extensions and another on neighborhoods---\,are useful
because they are not difficult requirements to verify and they imply facts about both the compatibility graph \emph{and} zipper constraints.
Still, they are fairly abstract. 
One might wonder,
for instance, whether
Definitions~\ref{def:globally-lang-comparable} and~\ref{def:neighborhood-comparable} 
really differ essentially. 
(They are distinct, as we will see shortly.) 

\subsection{Prior cases in the literature}
\label{section:prior}
We now use the conditions just introduced 
to
re-examine
three
sub-classes of filter 
for which
polynomial-time
minimization 
has been
reported in the literature.
This treatment provides a new understanding of the 
relationships between these 
special cases.
To start, each sub-class must be defined.

When discussing the fact that \fmp is \nphard, the authors in
\cite{okane17concise} point out that this may, at first, seem unexpected since
minimization of deterministic finite automata (DFA) is efficient (e.g.,
via the theorem of Myhill--Nerode\,\cite{hopcroft79introduction}).  
As an intuition for this difference, they offer the following perspective: 
when a sequence crashes on a
DFA, that string is outside of the automaton's language---whereas, when a
sequence crashes on a filter, \fmp allows the minimizer to select any output
for it, and some choices will likely give more compression than others. These 
degrees-of-freedom represent a combinatorial choice within the \fmp problem.

One way to curtail the explosion of such choices, then, is to
ensure that no strings can ever crash:
\arxivOrShort{
\begin{definition}[\nme filters\,\cite{saberifar17special}]
A \defemp{\nme filter} 
is a filter with the property that every state has an outgoing edge for every
observation in $Y$.
\end{definition}
}{
\begin{definition}[\cite{saberifar17special}]
In a \defemp{\nme filter},
every state has an outgoing edge for every
observation in $Y$.
\end{definition}
}
For any \nme filter $\struct{F}$, we have the language $\extension{\struct{F}}{v_0} = \KleeneStr{Y}$, 
i.e., the Kleene star of the set of observations.

The authors of \cite{saberifar17special} also identify a sort of obverse to the foregoing sub-class. If \nme
filters fully re-use observations, every state having all of them, consider next the
sub-class where observations occur at precisely one state:

\arxivOrShort{
\begin{definition}[\oao filters\,\cite{saberifar17special}]
In a \defemp{\oao filter}, each observation in $Y$ appears at most once. 
\end{definition}
}{
\begin{definition}[\cite{saberifar17special}]
A \defemp{\oao filter} is a filter where every observation in $Y$ appears at most once. 
\end{definition}
}
In a quite different context, the cardinality of the observation set was shown to affect
complexity; constraining this gives another sub-class of filters:

\arxivOrShort{
\begin{definition}[\un filters\,\cite{zhang22sso}]
A \defemp{\un filter} 
is a filter with a set of observations that is a singleton, i.e.,
$|Y| = 1$.
\end{definition}
}{
\begin{definition}[\cite{zhang22sso}]
A \defemp{\un filter} 
is a filter with a set of observations that is a singleton, i.e.,
$|Y| = 1$.
\end{definition}
}
Despite the fact that the preceding three sub-classes 
impose a diverse assortment of constraints, their common trait is efficiency:
any filter belonging to those sub-classes can be minimized in time that is polynomial in the input size~\cite[Thms.~3 and~4]{saberifar17special},
\cite[Thm.~3]{zhang22sso}.

The sub-class of \oao filters is disjoint from the \nme ones, with the sole
exception of filters with only a single state. Any \un filter is (exclusively)
either a linear chain, or includes a cycle.  In the latter case, every state
will have an outgoing edge, and that filter therefore has no edges missing.
A \un filter with multiple states can only have 
observations that appear once when it is a chain with $|V|=2$;
longer chains are neither \nme nor \oao filters.

\arxivOrShort{
\begin{figure}[t]
     \hspace*{.4cm}\includegraphics[width=0.8\linewidth]{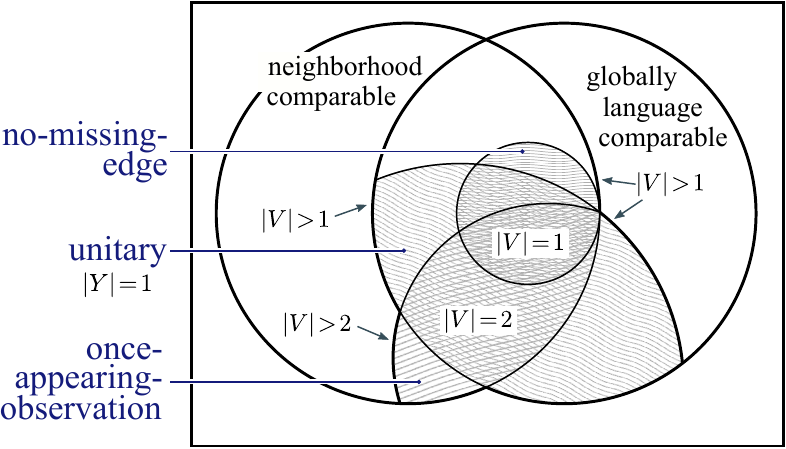}
	\caption{The relationships between the three sub-classes of filter in terms of the \glc and \nc properties identified. 
    (Note: Trivial filters, such as those with empty language, no vertices, and empty observation set have been omitted.) }
	\label{fig:venn}
\end{figure}

\begin{lemma}
A \nme filter is both \glc and \nc.
\end{lemma}
\begin{proof}
The first is straightforward: all pairs 
$v,w \in V(\struct{F})$ have  
$\extension{\struct{F}}{v} \subseteq \extension{\struct{F}}{w}$ as
every state has $\extension{\struct{F}}{v} = \KleeneStr{Y}$.
For the second: 
suppose $u,v,w$ are in $G_{\struct{F}}$ with ($i$)~$u \compatible v$ and ($ii$)~$v \compatible w$.
Then ($i$) means $u$ and $w$ both give the same output for every sequence $s \in \KleeneStr{Y}$. 
And ($ii)$ means the same for $v$ and $w$. But then, for every $s \in \KleeneStr{Y}$, 
 $u$ and $w$ must give outputs that are identical, thence
$u \compatible w$. Transitivity having been established, relation $\compatible$
is an equivalence for any \nme $\struct{F}$. 
\end{proof}

The remaining two sub-classes will aid in distinguishing the 
\glc property from the \nc property.

\begin{lemma}
A \oao filter is \nc; furthermore, non-trivial instances of such filters are not \glc.
\end{lemma}

\begin{proof}
Any two states in a \oao filter can only have the empty subsequence as a common extension. Thus, two states are 
compatible if and only if they have the same output, clearly defining equivalent classes of states, and thus
the \nc property holds. 
Any \oao filter with a pair of vertices $v$ and $w$, each bearing at least one outgoing edge, will have 
$\extension{\struct{F}}{v} \not\subseteq \extension{\struct{F}}{w}$ and 
$\extension{\struct{F}}{v} \not\supseteq \extension{\struct{F}}{w}$.
\end{proof}

\begin{lemma}
Any \un filter is \glc, but even simple instances of such filters may fail to be \nc.

\end{lemma}

\begin{proof}
If the filter includes a cycle, then every state $v$ has $\extension{\struct{F}}{v} = \KleeneStr{Y}$.
Otherwise, the language is finite, then ordering the states $v_0, v_1, v_2, \dots, v_{|V|}$ along the chain
gives that any pair $v_i$ and $v_j$, with $i \leq j$ have
$\extension{\struct{F}}{v_i} \supseteq \extension{\struct{F}}{v_j}$.
This \num{5} state \un filter 
\raisebox{-3pt}{\!\includegraphics[width=0.4\linewidth]{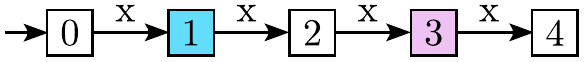}},
where colors represent outputs,
is an example that violates the requirements for \nc because 
$0 \compatible 4$ and $4 \compatible 2$, but $0 \not\compatible 2$.
\end{proof}

By way of summary, Figure~\ref{fig:venn} illustrates the relationships between
these filters and the properties of Section~\ref{subsec:handy}.  None of the sub-classes strictly subsumes the others,
yet 
common underlying properties explain
their
efficiency: ($i$)~their
compatibility graphs will be chordal and ($ii$)~their zipper constraints are tame.
The route by which these facts
are established depends upon the sub-class: using the \glc property, ($i$) comes
from Lemma~\ref{lm:glc_chordal} and ($ii$) from Lemma~\ref{lm:glc_to_cpzn};
using the \nc property, they come from Lemmas~\ref{lm:nc_chordal} and~\ref{lm:nc_to_cpzn}, respectively.
}
{
\begin{figure}[t]
\hspace*{.6cm}\includegraphics[width=0.8\linewidth]{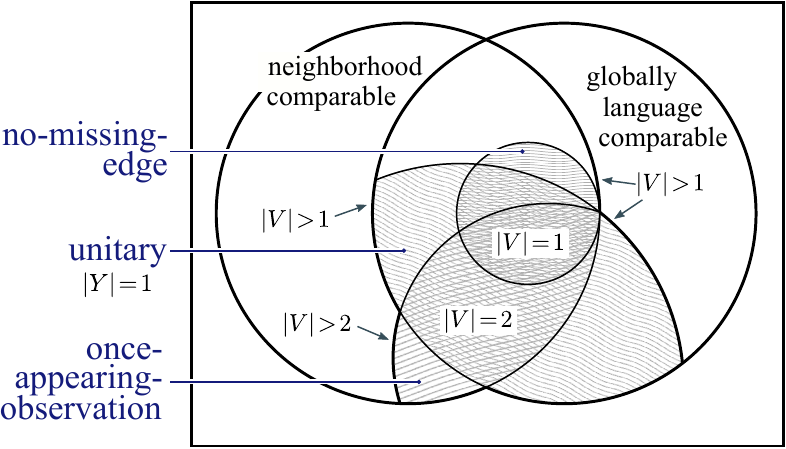}
	\caption{The relationships between the three sub-classes of filter in terms of the globally language and neighborhood comparable properties identified. 
	(Note: Degenerate filters are omitted.)
	\label{fig:venn}
    }
\end{figure}

Each of these three filters possess at least one of the properties introduced in Section~\ref{subsec:handy}.  
Along with their interrelationships, this
is illustrated in 
Figure~\ref{fig:venn}.
None of the sub-classes strictly subsumes the others,
yet 
common underlying properties explain
their
efficiency: ($i$)~their
compatibility graphs will be chordal and ($ii$)~their zipper constraints are tame.
The route by which these facts
are established depends upon the sub-class: using the \glc property, ($i$) comes
from Lemma~\ref{lm:glc_chordal} and ($ii$) from Lemma~\ref{lm:glc_to_cpzn};
using the \nc property, they come from Lemmas~\ref{lm:nc_chordal} and~\ref{lm:nc_to_cpzn}, respectively.
}

\subsection{A new instance now recognizable as efficiently solvable}
Thus, it has turned out that the compatibility graphs of previously known special
cases are all chordal. 
This gives a ready means by which extra \fmp instances can be 
identified 
as having efficient solutions, instances that have never 
previously been
recognized 
as such.
One basic example is shown in 
Figure~\ref{fig:perfect-solvable}.
This filter,
 $\struct{F}_{\fastsimple}$, has a
compatibility graph that fails to be chordal, but is a perfect graph. 
Additionally, 
$\struct{F}_{\fastsimple}\!$
possesses the \cpzn property because states \num{5}, \num{7}, \num{9} are all mutually compatible.
More generally, the perfect graphs include many non-chordal members, including bipartite graphs, Tur\'{a}n graphs, etc.\ 
Any filters with these as compatibility graphs (and
Theorem~\ref{thm:realizable_compatibility} says there \emph{are} such filters) are candidates as additional special cases over and above what
was previously known; one then needs to ensure their zipper constraints are benign: say, being $\emptyset$ or repairable. 


\setlength{\belowcaptionskip}{-1pt}
\begin{figure}[t]
     \centering
     \begin{subfigure}[b]{0.485\linewidth}
         \centering
         \includegraphics[width=\linewidth]{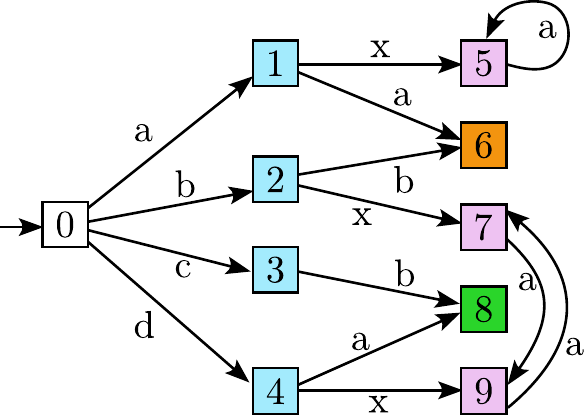}
         \caption{A simple \num{10}-state filter $\struct{F}_{\fastsimple}$, ~~\phantom{x} \phantom{.}~~~~with outputs visualized as colors.}
         \label{fig:perfect-filter}
     \end{subfigure}
     \hfill
     \begin{subfigure}[b]{0.48\linewidth}
         \centering
         \includegraphics[width=0.85\linewidth]{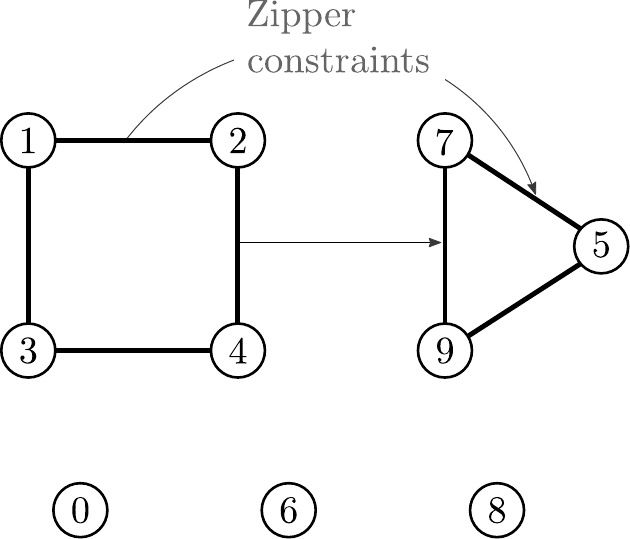}
         \caption{Compatibility graph $G_{\struct{F}_{\fastsimple}}$~~\phantom{.} \phantom{.}~~~~\,and zipper constraint set.}
         \label{fig:perfect-compatibility}
     \end{subfigure}
	\setlength{\belowcaptionskip}{-10pt}
	\caption{A filter now recognizable as efficiently minimizable.}
	\label{fig:perfect-solvable}
\end{figure}

%
%
%
%
\section{Conclusion}

This paper has identified basic properties underlying instances of filter
minimization that are easy to solve.  There are two aspects: the
structure of the compatibility graph and the determinism-enforcing zipper
constraints; they are distinct and both are shown to matter---as \nphard problems arise if a single
aspect is constricted but the other given free rein.  
Uncovering these facts, 
the key contribution of the first part of the paper, 
involves
distinguishing and formalizing several
subtle properties, proving new hardness results, and devising an efficient
algorithm for repair of violated zipper constraints.

The paper then turns to more pragmatic ways the preceding
insights can be leveraged:
as is usual with \nphard problems,  
researchers have sought conditions that identify sub-classes for 
which minimizers can be found efficiently.
The sub-class we identify subsumes the previously known special cases.
The paper further improves understanding of previously known
sub-classes, detailing
their differences and also drawing out commonalities.
Finally, the paper gives an example filter which, prior to here, would not 
be recognized as possessing an efficient solution.

\bibliographystyle{plain}
\bibliography{mybib}

\end{document}

%% file: defs.tex
\theoremstyle{definition} 
\usepackage{wasysym}

\newtheorem{x}{X}

\newtheorem{definition}[x]{Definition} 
\newtheorem{corollary}[x]{Corollary}

\newtheorem{lemma}[x]{Lemma} 
\newtheorem{theorem}[x]{Theorem} 
 
\newtheorem{property}[x]{Property}

\definecolor{grey}{rgb}{0.5,0.5,0.5}
\definecolor{darkgrey}{rgb}{0.15,0.15,0.15}
\definecolor{darkblue}{rgb}{0.05,0.05,0.5}
\definecolor{darkgreen}{rgb}{0.05,0.5,0.05}
\definecolor{darkestgreen}{rgb}{0.5,0.0,0.5}
\definecolor{darkorange}{rgb}{0.5,0.25,0.00}

\providecommand{\defemp}[1]{\emph{#1}} 
\newcounter{tecounter}
\providecommand{\extensions}[2]{\ensuremath{\mathcal{L}_{#1}(#2)}}

\providecommand{\extension}[2]{\ensuremath{\mathcal{L}_{#1}(#2)}}

\newcommand*{\probleminternal}[4]{
    {\small
	\par
	\medskip
	\noindent\fbox{\parbox{0.98\columnwidth}{
		\hspace*{-1.9em}\textbf{#4:} {#1} \\[0.05in]
		\renewcommand{\tabcolsep}{2pt}
		\begin{tabularx}{0.97\linewidth}{rX}
			~\emph{Input:} & #2 \\
			~\emph{Output:} & #3
		\end{tabularx}
	}}}
	\par
	\medskip
	\par
}

\newcommand*{\ourclass}[2]{
    {\small
	\par
	\bigskip
	\noindent\hfill\fbox{\parbox{0.95\columnwidth}{
		\textbf{#1:}\\[0.05in]
		 {#2} 
	}}\hfill}
	\par
	\bigskip
	\par
}

\newcommand*{\ourproblem}[3]{\probleminternal{#1}{#2}{#3}{\quad~ Problem}}

\newcommand{\KleeneStr}[1]{\ensuremath{{#1}^{\ast}}}
\providecommand{\compatible}{\ensuremath{\thicksim}}

\newcommand{\scalingsize}{0.85}
\newcommand{\fmp}{\scalebox{\scalingsize}{{\sc FM}}\xspace}
\newcommand{\fm}[1]{\scalebox{\scalingsize}{{\sc FM}}$({#1})$\xspace}
\newcommand{\mzccp}{\scalebox{\scalingsize}{\sc MZCC}\xspace}
\newcommand{\mzcc}[2]{\scalebox{\scalingsize}{\sc MZCC}$({#1},{#2})$\xspace}
\newcommand{\mcc}[1]{\mzcc{#1}{\emptyset}}

\newcommand{\chord}[1]{\ensuremath{{#1}^{\raisebox{1pt}{\scalebox{0.33}{$\bm\triangle$}}}}\xspace}
\newcommand{\neckl}[1]{\ensuremath{{#1}^{\raisebox{1pt}{\scalebox{0.33}{$\bm\triangle$}}}}\xspace}

\providecommand{\ptext}{{\footnotesize\textsf{P}}}
\providecommand{\nptext}{{\footnotesize\textsf{NP}}}

\providecommand{\p}{\ptext\xspace}

\providecommand{\nphard}{{\nptext -hard}\xspace}
\providecommand{\npcomplete}{{\nptext -complete}\xspace}

\providecommand{\struct}[1]{\ensuremath{\mathscr{#1}}\xspace}
\providecommand{\cover}[1]{\ensuremath{\mathbb{#1}}}

\providecommand{\neighbor}[2]{\ensuremath{\mathcal{N}_{#1}\left(#2\right)}}

\newcommand{\set}[1]{\mathbf{#1}}

\newcommand{\algo}[1]{\mathscr{#1}}

\newcommand\superrestr[3]{{
  \left.\kern-\nulldelimiterspace 
  #1 
  \vphantom{\big|} 
  \right|_{#2}^{#3} 
  }}

\newcommand\restr[2]{{
  \left.\kern-\nulldelimiterspace 
  #1 
  \vphantom{\big|} 
  \right|_{#2} 
  }}

\newcommand{\citep}[1]{\cite{#1}}
\newcommand{\citet}[1]{\cite{#1}}

\renewcommand{\emptyset}{\varnothing}

\makeatletter
\newcommand*\bigcdot{\mathpalette\bigcdot@{.7}}
\newcommand*\bigcdot@[2]{\mathbin{\vcenter{\hbox{\scalebox{#2}{$\m@th#1\bullet$}}}}}
\makeatother

\providecommand{\ptext}{{\footnotesize\textsf{P}}}
\providecommand{\nptext}{{\footnotesize\textsf{NP}}}

\providecommand{\p}{\ptext\xspace}

\providecommand{\nphard}{{\nptext -hard}\xspace}
\providecommand{\npcomplete}{{\nptext -complete}\xspace}

\providecommand{\cpzn}{comparable zip candidates\xspace}

\providecommand{\glc}{globally language comparable\xspace}
\providecommand{\Glc}{Globally language comparable\xspace}
\providecommand{\nc}{neighborhood comparable\xspace}
\providecommand{\Nc}{Neighborhood comparable\xspace}

\providecommand{\triv}{\text{\tiny triv}}

\providecommand{\ec}{efficiently coverable\xspace}
\providecommand{\ecadj}{efficiently coverable\xspace} 
\providecommand{\ecsimple}{\text{\tiny ec}}

\providecommand{\fastsimple}{\text{\tiny poly}}

\providecommand{\oao}{once-appearing-observation\xspace}
\providecommand{\nme}{no-missing-edge\xspace}

\providecommand{\un}{unitary\xspace}

\newcommand{\zwe}[2]{\ensuremath{w_{#1}^{\scalebox{0.67}{\,\ensuremath{#2}}}}}
\newcommand{\zweBig}[2]{\ensuremath{w_{#1}^{\scalebox{0.95}{\,\ensuremath{#2}}}}}

\newcommand\arxivOrShort[2]{\ifdefined\arxiv #1\else #2\fi}
\newcommand\shortvspace[1]{\ifdefined\arxiv\else \vspace*{#1}\fi}
